\documentclass{article} 
\usepackage{amsmath,epsfig}
\usepackage{algorithm,algorithmic,amsthm}

\newtheorem{theorem}{Theorem}[section]
\newtheorem{proposition}{Proposition}
\newtheorem{lemma}{Lemma}

\newtheorem{defNew}{Definition}

\def\X{{\mathbf X}}

\def\cT{{\mathcal T}}

\def\1{{\mathbf 1}}
\newcommand{\clust}{\mathcal{C}}

\title{Hierarchical Clustering using Randomly Selected Similarities}

\author{
Brian Eriksson \\
{Technicolor Research}\\
\texttt{brian.c.eriksson@gmail.com}
}

\begin{document}

\maketitle

\begin{abstract}
The problem of hierarchical clustering items from pairwise
similarities is found across various scientific disciplines, from biology to networking. Often, applications of clustering techniques are limited by the cost of obtaining similarities
between pairs of items.  While prior work has been developed to reconstruct clustering using
a significantly reduced set of pairwise similarities via adaptive measurements,
these techniques are only applicable when choice of similarities are available to the user. In
this paper, we examine reconstructing hierarchical clustering under similarity observations at-random.  We derive precise bounds which show that a significant fraction of the hierarchical clustering can be recovered using fewer than all the pairwise similarities.
We find that the correct hierarchical clustering down to a constant fraction of the total number of items ({\em i.e.,} clusters sized $O\left(N\right)$) can be found using only $O\left(N\log{N}\right)$ randomly selected pairwise similarities in expectation.
\end{abstract}

\section{Introduction}
\label{sec:intro}



Hierarchical clustering based on pairwise similarities arises
routinely in a wide variety of engineering and scientific problems.
These problems include inferring gene behavior from microarray data
\cite{genePaper}, Internet topology discovery \cite{yale}, detecting
community structure in social networks \cite{socioPaper},
advertising \cite{advertPaper}, and database management
\cite{Chaudhuri,Arasu}.  Often there is a significant cost
associated with obtaining each similarity value.  This cost can
range from computation time to calculate each pairwise similarity
({\em e.g.,} phylogenetic tree reconstruction using amino acid
sequences \cite{phylTrees}) to measurement load on the system under
consideration ({\em e.g.,} Internet topology discovery using
tomographic probes \cite{yale}).  In addition, situations where the
similarities require an expert human to perform the comparisons
results in a significant cost in terms of time and patience of the
user ({\em e.g.,} human perception experiments in~\cite{advertPaper}).

Prior work has attempted to develop efficient hierarchical
clustering methods \cite{Hofmann98activedata,Grira08}, but
many proposed techniques are heuristical in nature and do not provide
any theoretical guarantees.  Derived bounds are available in \cite{eriksson11} which robustly finds clusters down to $O\left(\log{N}\right)$.  The main limitation to this approach is that it is only
applicable for problems where the user has control over the specific
pairs of items to query and when the pairwise similarities are
acquired in an online fashion ({\em i.e.,} one at a time, using past
information to inform future samples). In many situations, either
this control is not available to the user, or a subset of similarities are
acquired in a batch setting ({\em e.g.,} recommender systems problems~\cite{netflixPrize}).  This
motivates resolving the hierarchical
clustering given a selected number of pairwise similarities observed
at-random, where adaptive control is not available.  Specifically, we look to answer the following question: {\em How many similarities observed at-random are required to reconstruct the hierarchical clustering?}  The work in \cite{robustAlg} developed a novel clustering technique to approximately resolve clusters down to size $O\left(N\right)$ from random observations, in contrast we look to reconstruct the clustering hierarchy (to some pruning) exactly with high probability and also examine discovering clusters of size significantly less than $O\left(N\right)$.

While results in \cite{eriksson11} indicate that resolving the {\em entire} hierarchical clustering using a sampling at-random regime requires effectively all the pairwise similarities, we find that a significant fraction of the clustering hierarchy can be resolved accurately.  Specifically, we resolve the similarity sampling rate required given a desired level of clustering resolution.  The only restriction we will place on the observed pairwise similarities are that they satisfy the {\em Tight Clustering} (TC) condition, which states that intracluster similarity values are greater than intercluster similarity values.  This sufficient condition is required for any minimum-linkage clustering procedure, and can commonly be found underlying branching processes  where the similarity between items is a monotonic increasing function of the distance from the branching root to their nearest common branch point (such as clustering resources in the Internet \cite{treeNess}), or when similarity is defined by density-based
distance metrics \cite{Sajama}.

Our results show that the hierarchy down to clusters sized a constant fraction of the total number of items ({\em i.e.,} $O\left(N\right)$) can be resolved with only $O\left(N\log{N}\right)$ pairwise similarities observed at-random on average.  To find smaller clusters of size $O\left(N^\beta\right)$, where $\beta\in\left(0,1\right)$, we derive bounds which show that only $O\left(N^{2-\beta}\log{N}\right)$ similarities are required in expectation.

The paper is organized as follows.  The hierarchical clustering
problem and problems associated with randomly selected pairwise
similarities are introduced in Section~\ref{sec:hier}.  Our derived bounds are presented in
Section~\ref{sec:method}. Finally, concluding remarks are made in Section~\ref{sec:summary}.

\section{Hierarchical Clustering and Notation}
\label{sec:hier}




Let $\X = \{x_1,x_2,\ldots,x_N\}$ be a collection of $N$ items which has an underlying {\em hierarchical clustering} denoted as ${\cal T}$.
\begin{defNew}
A \textbf{cluster} $\clust$ is defined as any subset of $\X$. A
collection of clusters $\cT$ is called a \textbf{hierarchical
clustering} if $\cup_{\clust_i\in\cT}\clust_i = \X$ and for any
$\clust_i,\clust_j \in \cT$, only one of the following is true
\textbf{(i)} $\clust_i\subset\clust_j$, \textbf{(ii)} $\clust_j
\subset\clust_i$, \textbf{(iii)} $\clust_i\cap\clust_j = \emptyset$.
\end{defNew}
Without loss of generality, we will consider $\cT$ as a complete binary tree, with $N$ leaf nodes and where every $\clust_k\in\cT$ that is not a leaf of the tree, there
exists proper subsets $\clust_i$ and $\clust_j$ of $\clust_k$, such
that $\clust_i\cap\clust_j = \emptyset$, and $\clust_i\cup\clust_j =
\clust_k$.

Our measurements will be from ${\bf S} = \{s_{i,j}\}$ the collection of all pairwise
similarities between the items in ${\bf X}$, with $s_{i,j}$ denoting
the similarity between $x_i$ and $x_j$ and assuming
$s_{i,j}=s_{j,i}$.  The similarities must conform to the hierarchy of $\cT$ through the following sufficient condition.
\begin{defNew}
\label{def:cl} The triple $({\bf X},\mathcal{T},{\bf S})$ satisfies
the \textbf{Tight Clustering (TC) Condition} if for every set of
three items $\{x_i,x_j,x_k\}$ such that $x_i,x_j \in \clust$ and
$x_k \not\in\clust$, for some $\clust \in \cT$, the pairwise
similarities satisfies, $s_{i,j} >
\max\left(s_{i,k},s_{j,k}\right)$.
\end{defNew}

In words, the TC condition implies that the similarity between all
pairs within a cluster is greater than the similarity with respect
to any item outside the cluster.  Under the TC condition, the tree found by agglomerative clustering will match the true clustering hierarchy, ${\cal T}$.  Minimum-linkage agglomerative clustering~\cite{statLearning} is a recursive process that begins with singleton
clusters ({\em i.e.,} the $N$ individual items to be clustered). At
each step of the algorithm, the pair of most similar clusters
associated with the largest observed pairwise similarity are merged.
The process is repeated until all items are merged into a single
cluster.  The main drawback to this technique is that it requires knowledge of all $O\left(N^2\right)$ pairwise similarities value ({\em i.e.,} all values must be known to find the maximum), therefore this methodology will be infeasible for problems where $N$ is large, or where there is a significant cost to obtaining each similarity.


To reduce the measurement cost, we consider an incomplete observation of pairwise similarities.  For our specific model, we define the indicator
matrix of similarity observations, $\Omega$, such that
$\Omega_{i,j}=1$ if the pairwise similarity $s_{i,j}$ has been
observed and $\Omega_{i,j}=0$ if the pairwise similarity $s_{i,j}$
is not observed ({\em i.e.,} unknown). The pairwise similarities are observed {\em uniformly at-random}, defining the similarity observation matrix as,
\begin{eqnarray}
\begin{array}{cc}
P\left(\Omega_{i,j}=1\right) = p & : \forall{i,j}
\end{array}
\end{eqnarray}
For some probability, $p>0$.



To reconstruct the clustering from incomplete measurements, this paper focuses on a slightly modified version of
minimum-linkage agglomerative clustering.  This process can be considered
off-the-shelf agglomerative clustering where the pairwise
similarities not observed are simply ignored ({\em i.e.,} unobserved similarities are zero-filled).  The methodology is described in
Algorithm~\ref{alg:incompleteAgg}.

\begin{algorithm}[h]
\caption{- {\tt incompleteAgglomerative}$\left({\bf X}, {\bf S},
\Omega\right)$}\label{alg:incompleteAgg} {\textbf {Given :}}
\begin{enumerate}
\item Set of items, ${\bf X} = \{x_1,x_2,...x_N\}$.
\item Matrix of pairwise similarities, ${\bf S}$.
\item Indicator matrix of observed pairwise similarities, $\Omega$.
\end{enumerate}
{\textbf {Clustering Process : }}
\begin{description}
\item Initialize clustering tree ${\cal T} =
\{{\cal C}_1,{\cal C}_2,...,{\cal C}_N\}$, where ${\cal C}_i = x_i$ for all $i=\{1,2,...N\}$.
\item Zero-fill unobserved pairwise similarities, $S_{i,j} = 0$ if $\Omega_{i,j} = 0$.
\item {\textbf {While}} $\max_i\left|{\cal C}_i\right|<N$
\begin{enumerate}
\item Find ${\widehat{i},\widehat{j}} = \mbox{arg}\max_{i,j}\left(S_{i,j}\right)$
\item Merge the items ${\widehat{i},\widehat{j}}$ in the  hierarchical clustering ${\cal T}$, and matrix of
pairwise similarities ${\bf S}$.
\begin{enumerate}
\item Set $S_{\widehat{i},k} = \max\{S_{\widehat{i},k},S_{\widehat{j},k}\}$ for all $k=\{1,2,...,N\}$.
\item Set $S_{\widehat{j},k} = 0$ for all $k=\{1,2,...,N\}$.
\item Set ${\cal C}_{\widehat{i}} = \{{\cal C}_{\widehat{i}}, {\cal C}_{\widehat{j}}\}$
\item Set ${\cal C}_{\widehat{j}} = \{\}$
\end{enumerate}
\end{enumerate}
\end{description}
{\textbf {Output : }}
\begin{description}
\item Return the estimated hierarchical clustering, ${\cal T}$.
\end{description}
\end{algorithm}


\subsection{Canonical Subproblem}

The intuition for why an incomplete subset of pairwise similarities is useful can be seen when considering the canonical subproblem where a single cluster ${\cal C}$ is split into two subclusters ${\cal C}_L,{\cal C}_R$ (where ${\cal C}={\cal C}_L\bigcup{\cal C}_R$ and ${\cal C}_L\bigcap{\cal C}_R=\emptyset)$.  In order to properly resolve the two clusters, it is necessary for the agglomerative clustering algorithm to have enough pairwise similarities to make an informed decision as to which items belong to which subcluster.  We define a sampling graph, ${\cal G}$, resolved from the pairwise similarity observation matrix $\Omega$.

\begin{defNew}
Consider the $N \times N$ pairwise similarity observation matrix $\Omega$, then the \textbf{sampling graph} ${\cal G}=\{{\bf V},{\bf E}\}$ is defined as graph with $\left|{\bf V}\right|=N$ nodes where the edge $e_{i,j}=1$ if $\Omega_{i,j}=1$ ({\em i.e.,} the pairwise similarities was observed) and $e_{i,j}=0$ otherwise ({\em i.e.,} the pairwise similarity was not observed).
\end{defNew}


In the context of the sampling graph, ${\cal G}$, we can state the following proposition with respect to resolving the canonical subproblem.

\begin{proposition}
\sloppypar
\label{prop:path} Consider a cluster ${\cal C}$ consisting of two subclusters, ${\cal C}_L$ and ${\cal C}_R$ (such that ${\cal C}_L \bigcup {\cal C}_R = {\cal C}$ and ${\cal C}_L \bigcap {\cal C}_R = \emptyset$). Then, the agglomerative clustering algorithm will resolve the two subclusters if and only if the sampling subgraphs associated with each subcluster ({\em i.e.,} ${\cal G}_L$ for cluster ${\cal C}_L$ and ${\cal G}_R$ for cluster ${\cal C}_R$) are both connected.
\end{proposition}

The intuition behind this proposition is as follows.  Any clustering procedure requires enough information to associate each item with the cluster it belongs to.  For example, minimum-linkage agglomerative clustering would require that each item observe at least a single similarity with another item in that cluster.  But this is not enough, as we also require that one of these two items have an observed similarity with another item in the remainder of the cluster ({\em i.e.,} one of the items in the cluster, not including the two items paired together), and so on until the all the items can be clustered.  In terms of the sampling graph, this is the requirement that a path can be found between items in the same cluster (as all of these pairwise similarities will be greater than any other item outside of this cluster, as stated using the TC condition).  It is then obvious that a cluster of items will only be returned if the sampling graph is connected between those items.  An example of this clustering can be found in Figure~\ref{fig:connectedTree}.

\begin{figure*}[htb]
\begin{minipage}[h]{0.32\linewidth}
   \centerline{\epsfig{figure=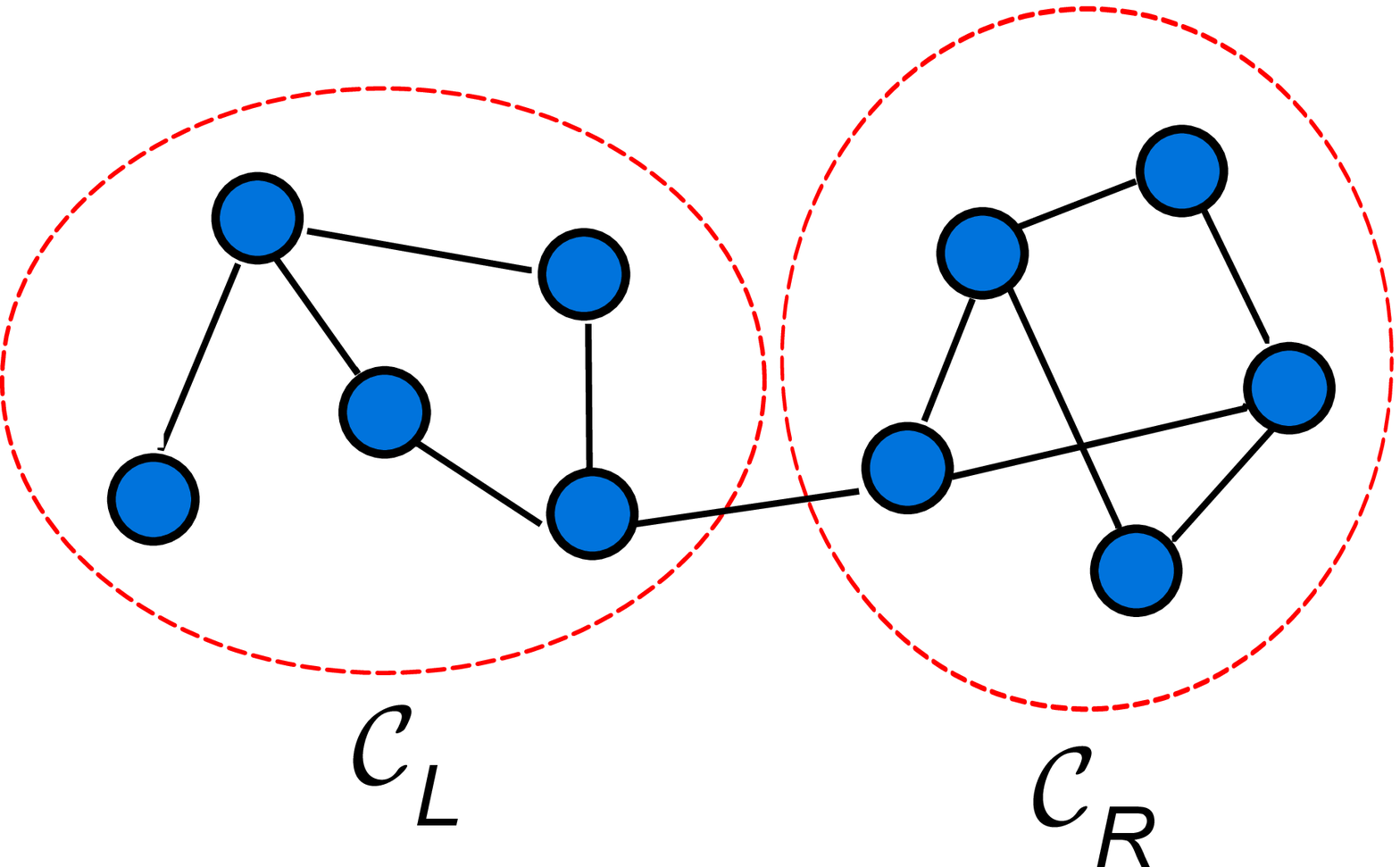,width=4.5cm}}
  \centerline{\textbf{(A)}}
\end{minipage}
\begin{minipage}[h]{0.32\linewidth}
  \centerline{\epsfig{figure=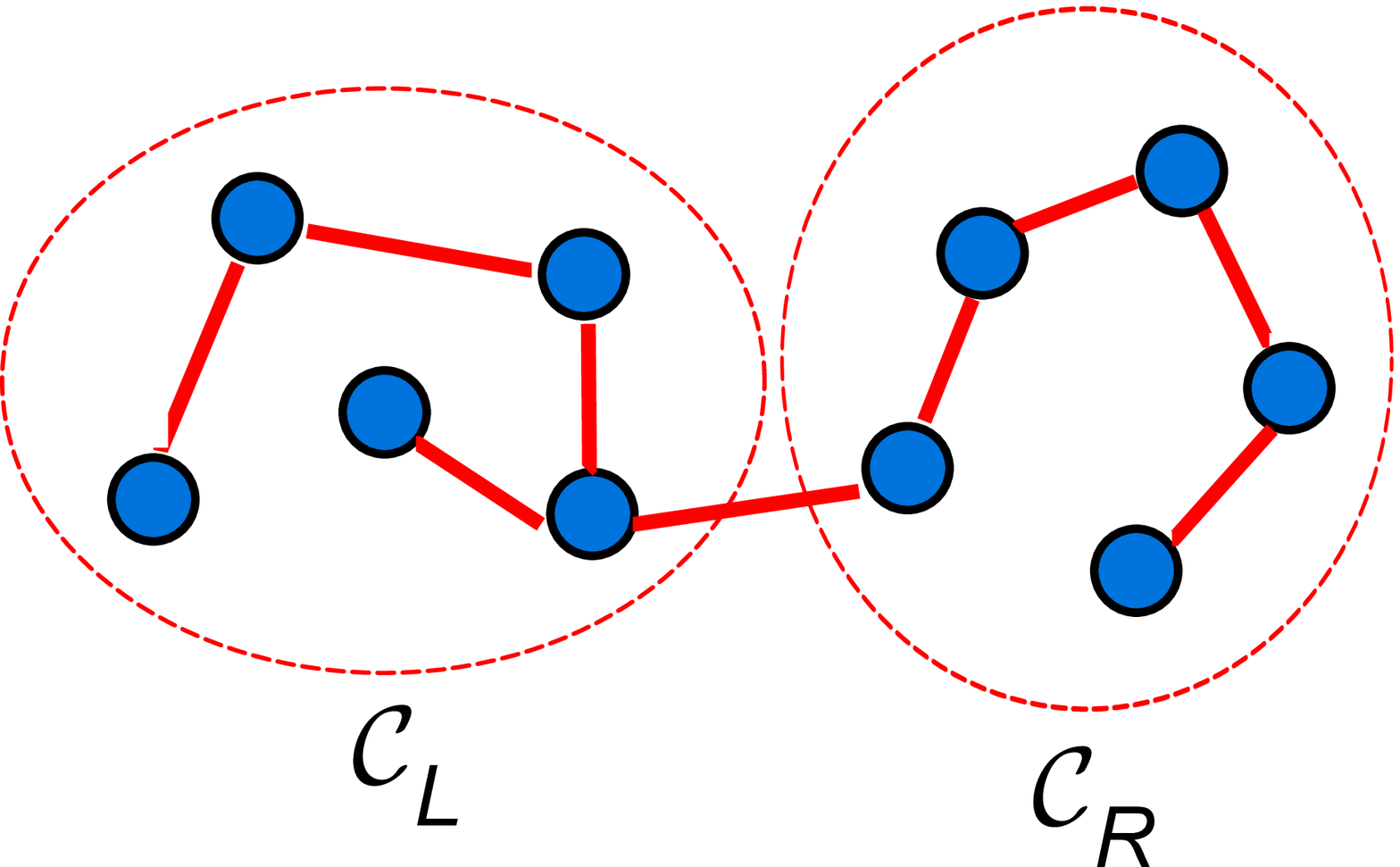,width=4.5cm}}
  \centerline{\textbf{(B)}}
\end{minipage}
\begin{minipage}[h]{0.32\linewidth}
  \centerline{\epsfig{figure=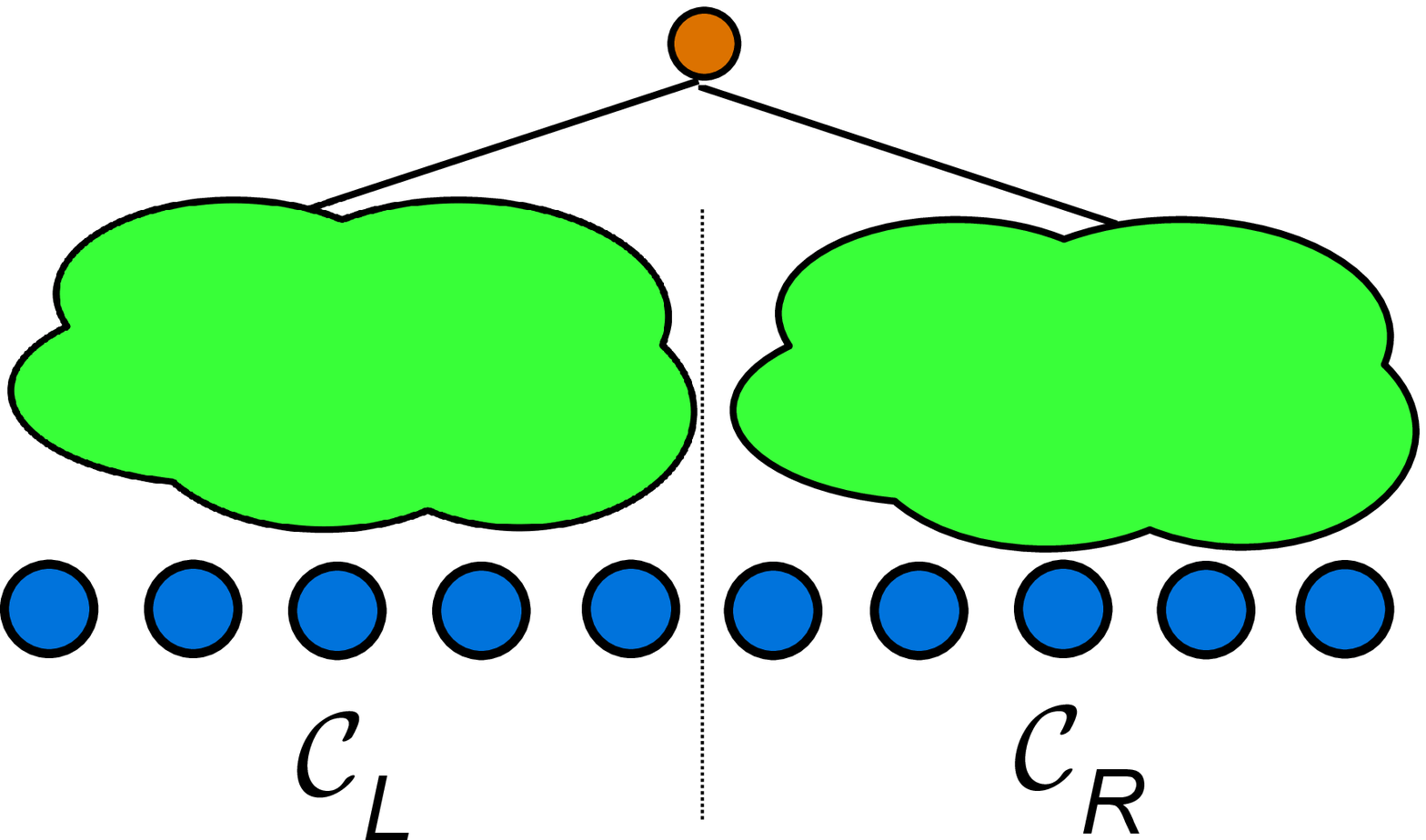,width=4.5cm}}
  \centerline{\textbf{(C)}}
\end{minipage}
\caption{\label{fig:connectedTree} (A) - Dense sampling graph on two clusters ${\cal C}_L,{\cal C}_R$, where an edge between nodes indicates that we have observed the pairwise similarity between those items, (B) - Example connected path found through the sampling graph using agglomerative clustering, (C) - Resulting correct hierarchical clustering.}
\end{figure*}

Alternatively, if a cluster of items is disconnected, into two sampling graph connected components ${\cal C}_A,{\cal C}_B$ (where ${\cal C}_A \bigcup {\cal C}_B = {\cal C}$ and ${\cal C}_A \bigcap {\cal C}_B = \emptyset$), then the clustering procedure will not have enough information to merge the items into a single cluster.  An example of this incorrect clustering can be found in Figure~\ref{fig:disconnectedTree}.

\begin{figure*}[htb]
\begin{minipage}[h]{0.32\linewidth}
   \centerline{\epsfig{figure=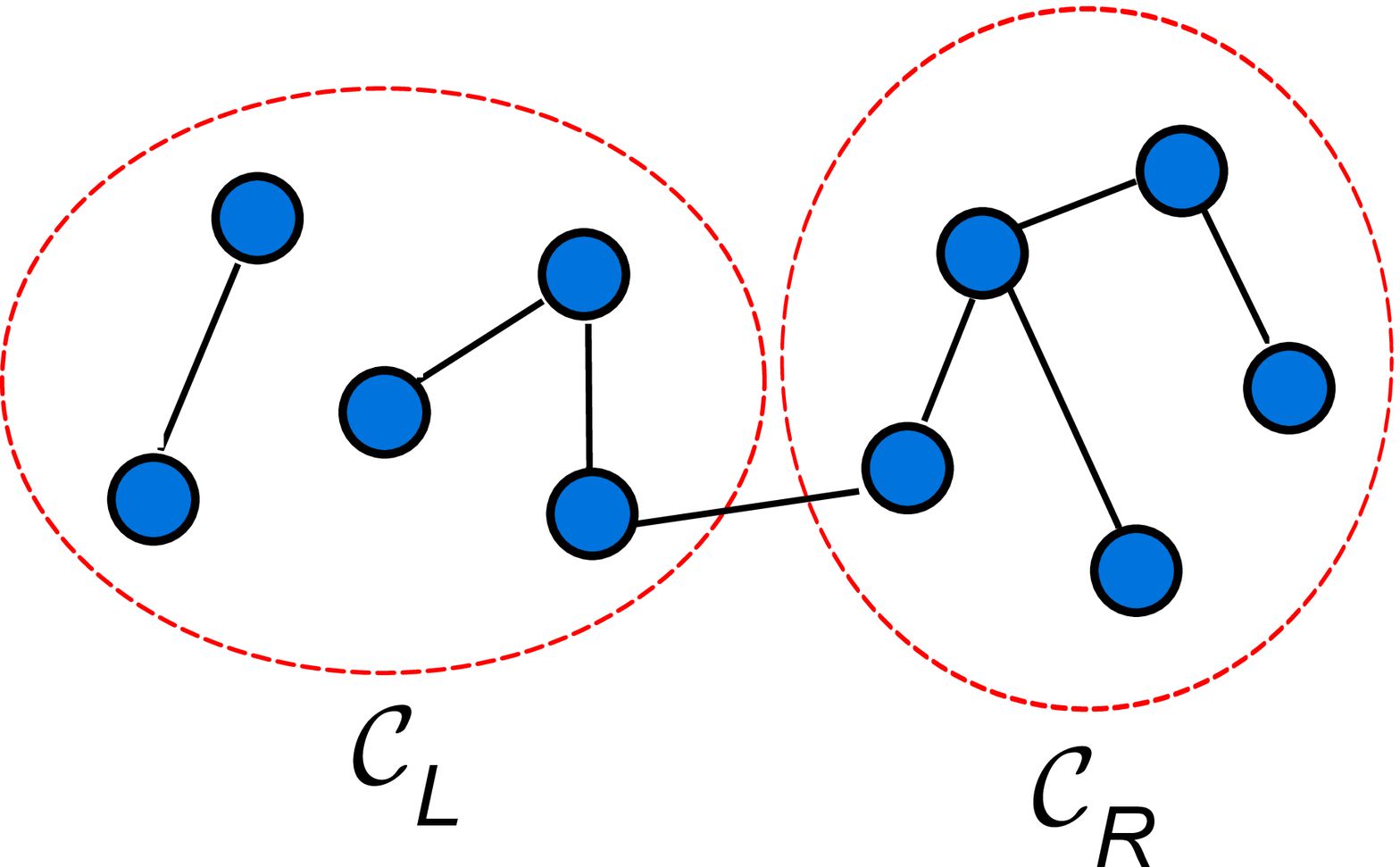,width=4.5cm}}
  \centerline{\textbf{(A)}}
\end{minipage}
\begin{minipage}[h]{0.32\linewidth}
  \centerline{\epsfig{figure=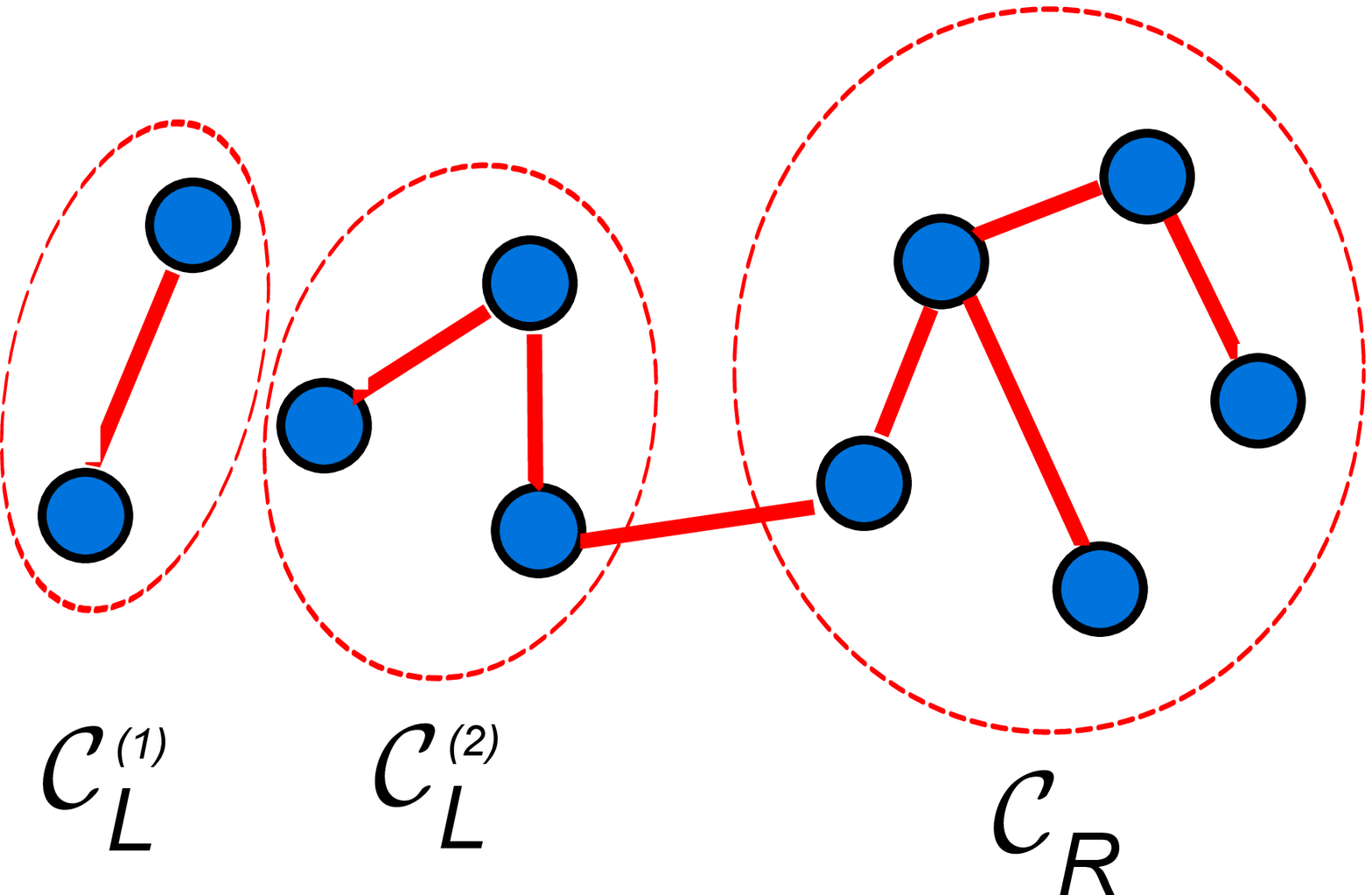,width=4.5cm}}
  \centerline{\textbf{(B)}}
\end{minipage}
\begin{minipage}[h]{0.32\linewidth}
  \centerline{\epsfig{figure=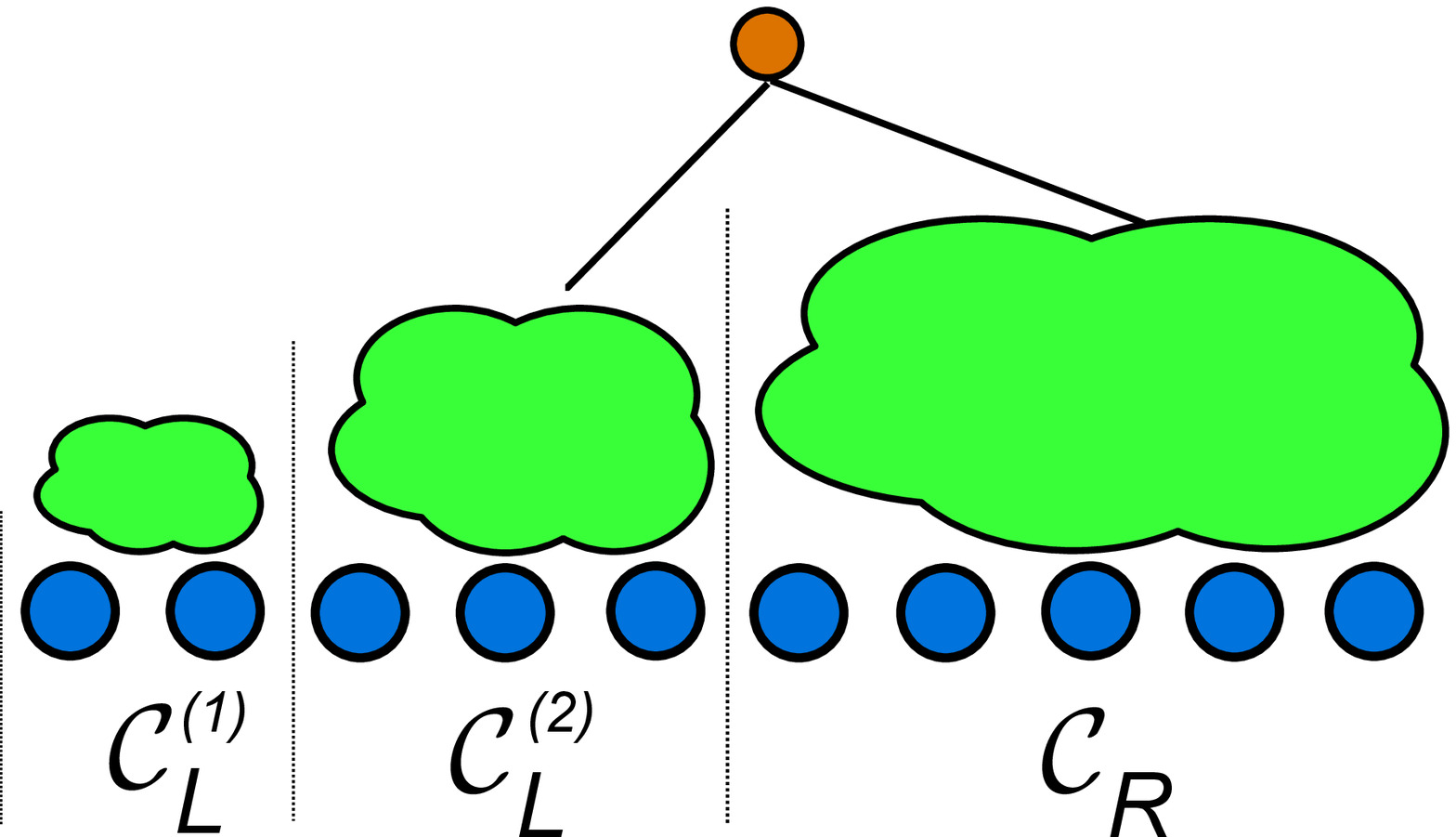,width=4.5cm}}
  \centerline{\textbf{(C)}}
\end{minipage}
\caption{\label{fig:disconnectedTree} (A) - Sparse sampling graph on two clusters ${\cal C}_L,{\cal C}_R$, where an edge between nodes indicates that we have observed the pairwise similarity between those items, (B) - Example disconnected path found through the sampling graph using agglomerative clustering, (C) - Resulting incorrect hierarchical clustering.}
\end{figure*}

\section{Main Results}
\label{sec:method}
When pairwise similarities are observed uniformly at-random, such that each pairwise similarity is observed with probability $p$, the resulting sampling graph can be considered a bernoulli random graph (where each edge exists with probability $p$).  Using Proposition~\ref{prop:path} and prior work on random graph theory \cite{gilbert}, we can state the following theorem.

\begin{theorem}
\sloppypar
\label{thm:global} Consider the quadruple $({\bf X},\mathcal{T},{\bf
S},\Omega\left(p\right))$, where the
Tight Clustering (TC) condition is satisfied, and $\mathcal{T}$ is
a complete (possible unbalanced) binary tree that is unknown. Then,
the agglomerative clustering algorithm recovers all
clusters of size $\geq n \geq 4$ of $\mathcal{T}$ with probability
$\geq\left(1-\alpha\right)$ for $\alpha>0$ given sampling $\Omega\left(p\right)$ satisifies,
\begin{eqnarray}
\label{eqn:completeBound}
p\geq \max\left\{1-\left(\frac{\alpha n}{78N}\right)^{2/n}, 1-\left(2^{\frac{1}{{n-1}}} - 1 \right)^{2/n}, 1-\left(\frac{\alpha}{2\left(N-n\right)}\right)^{1/n}\right\}
\end{eqnarray}
\end{theorem}
\begin{proof}
The proof of this theorem follows from the results of Propositions~\ref{prop:intraCluster} and~\ref{prop:interCluster}.
\end{proof}

The first component of Theorem~\ref{thm:global} requires that the sampling probability is large enough that a path in the sampling graph can be found with high probability for any collection of items of size $\geq n$.
\begin{proposition}
\label{prop:intraCluster} Given a set of $N\geq 4$ items, then
the agglomerative clustering algorithm will recover all $\leq \frac{N}{n}$ clusters of size $\geq n \geq 4$ of $\mathcal{T}$ with probability
$\geq\left(1-\frac{\alpha}{2}\right)$ given the sampling probability satisfies,
\begin{eqnarray}
p \geq \max\left\{1-\left(\frac{\alpha n}{78N}\right)^{2/n}, 1-\left(2^{\frac{1}{n-1}} - 1 \right)^{2/n}\right\}
\end{eqnarray}
\end{proposition}
\begin{proof}
We prove this proposition using prior work on random graph theory in the Appendix (Section~\ref{sec:appendixB}).
\end{proof}


While Proposition~\ref{prop:intraCluster} ensures that there are enough pairwise similarities to determine each leaf cluster (down to size $n$), to resolve the entire tree structure down to clusters of size $\geq n$ we additionally require enough similarities to determine the connectivity between these clusters.
\begin{proposition}
\label{prop:interCluster}
Consider the quadruple $({\bf X},\mathcal{T},{\bf
S},\Omega\left(p\right))$, where the
Tight Clustering (TC) condition is satisfied, and $\mathcal{T}$ is
a complete (possible unbalanced) binary tree that is unknown. Then,
given a set of clusters of size $\geq n$, the clustering structure of $\mathcal{T}$ pruned to cluster size $n$ will be resolved with probability
$\geq\left(1-\frac{\alpha}{2}\right)$ given sampling $\Omega\left(p\right)$ satisfies,
\begin{eqnarray}
p \geq 1-\left(\frac{\alpha}{2\left(N-n\right)}\right)^{1/N}
\end{eqnarray}
\end{proposition}
\begin{proof}
Consider the clustering structure of $\mathcal{T}$ and a single cluster of size $n$.  At most, there will be $N-n$ other clusters in $\mathcal{T}$ that must be compared against to construct the clustering hierarchy.  Given sampling rate $p$, then at least one item (out of $\geq n$) must observe at least one pairwise similarity with at least one item in the other cluster.  Therefore, to ensure that every cluster satisfies this with probability $\geq\left(1-\frac{\alpha}{2}\right)$, using the union bound we require the sampling rate to satisfy,
\begin{eqnarray*}
\left(1-p\right)^n &\leq& \frac{\alpha}{2\left(N-n\right)} \\
p&\geq& 1-\left(\frac{\alpha}{2\left(N-n\right)}\right)^{1/n}
\end{eqnarray*}
\end{proof}

Combining the results of Propositions~\ref{prop:intraCluster} and~\ref{prop:interCluster}, we find the sampling probability rate necessary to ensure with high probability ({\em i.e.,} $\geq 1 - \alpha$) that all the clusters of size $\geq n$ will be resolved, and the clustering hierarchy between these clusters can be reconstructed.  This is shown in Equation~\ref{eqn:completeBound}.

\subsection{Sampling Rate Required for Given Cluster Sizes}

Using the results from Theorem~\ref{thm:global}, we can state the expected number of pairwise similarity measurements needed to observe clusters down to a specified level.  For clusters of size $\sim O\left(\delta N^{\beta}\right)$, where $\beta\in\left(0,1\right]$ and $\delta\in\left(0,1\right]$, we find the following,

\begin{theorem}
\sloppypar
\label{thm:constantThm} Consider the quadruple $({\bf X},\mathcal{T},{\bf
S},\Omega\left(p\right))$, where the
Tight Clustering (TC) condition is satisfied, and $\mathcal{T}$ is
a complete (possible unbalanced) binary tree that is unknown. Then,
agglomerative clustering recovers all clusters of size $\geq \delta N^{\beta} \geq 4$ (where $\beta\in\left(0,1\right]$ and $\delta\in\left(0,1\right]$) of $\mathcal{T}$ with probability
$\geq\left(1-\alpha\right)$ given that the total number of items $N\geq\max\{4,\left(\frac{\alpha}{2}\right)^{1/\left(1-2\kappa\right)},\left(\frac{\alpha \delta}{78}\right)^{\frac{1}{\left(1-\beta-\kappa\right)}}\}$ and the pairwise similarities are sampled at-random with probability,
\begin{eqnarray}
p \geq \frac{2 \kappa}{\delta}N^{-\beta}\log{N}
\end{eqnarray}
For a constant oversampling factor $\kappa\geq 3$.
\end{theorem}
\begin{proof}
The proof of this theorem can be found in the Appendix (Section~\ref{sec:appendixB}).
\end{proof}

To see the improvements of these bounds, consider the following simple example.  We want to reconstruct the hierarchical clustering containing $N=1000$ items, specifically recovering all clusters of size $\geq 75$ (using $\delta = 0.5$, $\beta = 0.66$, and $\delta N^\beta = 75$) with probability $\geq \left(1-\alpha\right) = 0.95$.  Given oversampling factor $\kappa = 3$ and using the results of Theorem~\ref{thm:constantThm}, we find that to resolve this resolution of clustering only requires a similarity sampling rate of $p\geq 0.5526$, on average observing $276,020$ pairwise similarities.  This is a significant savings over standard techniques that require the entire set of $499,500$ similarities.

And finally we consider the ability to find large clusters of size $O\left(N\right)$.

\begin{theorem}
\sloppypar
\label{thm:logThm} Consider the quadruple $({\bf X},\mathcal{T},{\bf
S},\Omega\left(p\right))$, where the
Tight Clustering (TC) condition is satisfied, and $\mathcal{T}$ is
a complete (possible unbalanced) binary tree that is unknown. Then,
agglomerative clustering recovers all clusters of size $\geq \delta N  \geq 4$ (where $\delta\in\left(0,1\right]$) of $\mathcal{T}$ with probability
$\geq\left(1-\alpha\right)$ given that the number of items $N\geq\max\{4,\left(\frac{\alpha}{2}\right)^{1/\left(1-2\kappa\right)},\left(\frac{\alpha \delta}{78}\right)^{\frac{1}{-\kappa}}\}$ and the pairwise similarities are sampled at-random with probability,
\begin{eqnarray}
p \geq \frac{2 \kappa}{\delta}\frac{\log{N}}{N}
\end{eqnarray}
For a constant oversampling factor $\kappa\geq 3$.
\end{theorem}
\begin{proof}
The proof follows from Theorem~\ref{thm:constantThm}.
\end{proof}
Therefore, we find that to resolve clusters down to size $\geq \delta N \geq 4$ (for $\delta \in \left(0,1\right]$) requires only $\frac{2 \kappa}{\delta}N\log{N}$ randomly chosen pairwise similarities in expectation.  For example, to resolve the hierarchical clustering down to clusters of size $\geq 100$ from a set of $1000$ items requires (on average) less than $42\%$ of the complete set of pairwise similarities to be observed at-random (given $\alpha=0.05$).

%



\section{Conclusions}
\label{sec:summary}

Hierarchical clustering from pairwise similarities is found in disparate problems ranging from Internet topology discovery to bioinformatics.  Often these applications are limited by a significant cost required to obtain each pairwise similarity.  Prior work on efficient clustering required an adaptive regime where targeted measurements were acquired one-at-a-time.  In this paper, we consider the more general problem of clustering from a set of incomplete similarities taken at-random.  We present provable bounds demonstrating that resolving large clusters can be determined with only $O\left(N\log{N}\right)$ similarities on average.  Future work in this area includes developing efficient clustering techniques that are also robust to outlier measurements and considering alternative sampling methodologies.

\section{Appendix}

\subsection{Lemma~\ref{lemma:middle}}
\label{sec:appendixA}
\begin{lemma}
\label{lemma:middle}
Given $n\geq 4$ and $0\leq q \leq 1$,
\begin{eqnarray*}
2q^{n/2}\left(1+q^{\frac{n-2}{2}}\right)^{n-1} \leq 20q^{n/2}\left(1+q^{n/2}\right)^{n-1}
\end{eqnarray*}
\end{lemma}
\begin{proof}
Consider some constant value $\lambda\geq 1$, then we want to show that,
\begin{eqnarray*}
2q^{n/2}\left(1+q^{\frac{n-2}{2}}\right)^{n-1} \leq 2\lambda q^{n/2}\left(1+q^{n/2}\right)^{n-1}
\end{eqnarray*}

Rearranging both sides and given that $\lambda^{\frac{1}{n-1}} \in \left(1,\lambda\right)$ for $n>1$.,
\begin{eqnarray*}
{\left(1+q^{\frac{n-2}{2}}\right)} &\leq& \lambda^{\frac{1}{n-1}} +q^{n/2}\\
\end{eqnarray*}

Further rearranging and bounding the log function,
\begin{eqnarray*}
\log\left(1+q^{\frac{n-2}{2}}-q^{n/2}\right) &\leq& \log\left(\lambda^{\frac{1}{n-1}}\right)\\
q^{\frac{n-2}{2}}-q^{n/2} &\leq& \frac{1}{n-1}\log\left(\lambda\right)\\
q^{n/2}\left(\frac{1}{q}-1\right) &\leq& \frac{1}{n-1}\log\left(\lambda\right)\\
\end{eqnarray*}
If $q>\frac{1}{2}$, then this inequality is satisfied as the left-hand term is always negative and the right-hand term is always positive.

Considering $q\leq\frac{1}{2}$,
\begin{eqnarray*}
q^{n/2}\left(\frac{1}{q}-1\right) \leq q^{\frac{n-2}{2}} &\leq& \frac{1}{n-1}\log\left(\lambda\right)
\end{eqnarray*}

Taking the log of both sides,
\begin{eqnarray*}
{\frac{n-2}{2}}\log{q} &\leq& \log\log\left(\lambda\right) - \log\left(n-1\right)
\end{eqnarray*}

If $n\geq 4$, then $\left(\frac{n}{2} - 1\right) \geq \frac{n}{4}$.  Therefore,
\begin{eqnarray*}
{\frac{n}{4}}\log{q} &\leq& \log\log\left(\lambda\right) - \log\left(n\right)\\
\end{eqnarray*}

If $q\leq\frac{1}{2}$, then $\log{q}\leq\log{\frac{1}{2}}$, therefore,
\begin{eqnarray*}
{\frac{n}{4}}\log{\frac{1}{2}} - \log\log\left(\lambda\right) &\leq&  - \log\left(n\right)\\
n + \frac{4}{\log{2}}\log\log\left(\lambda\right) &\geq&  \frac{4}{\log{2}}\log\left(n\right)\\
\end{eqnarray*}

Solving numerically, we find that this is satisfied for all $n\geq 4$ if $\lambda=10$. This proves the result.

\end{proof}

\subsection{Proof of Proposition~\ref{prop:intraCluster}}
\label{sec:appendixB}
We begin by determining the sampling probability, $p$, necessary to ensure that a single set of $n$ items can be clustered.  This is equivalent to a bernoulli random graph (${\cal G}_{n,p}$) of size $n$ and probability $p$ being connected.  From \cite{gilbert}, we bound this probability as,
\begin{eqnarray*}
P\left({\cal G}_{n,p} \mbox{ is connected}\right) \geq 1-q^{n-1}\left(\left(1+q^{\frac{n-2}{2}}\right)^{n-1}-q^{\frac{\left(n-1\right)\left(n-2\right)}{2}}\right)-q^{n/2}\left(\left(1+q^{\frac{n-2}{2}}\right)^{n-1}-1\right)
\end{eqnarray*}
Where $q = 1-p$.
Simplifying in terms of $n$, and using $q^{n-1} \leq q^{n/2}$ (for all $n\geq 2$),
\begin{eqnarray*}
P\left({\cal G}_{n,p} \mbox{ is connected}\right) &\geq& 1-\left(1+q^{\frac{n-2}{2}}\right)^{n-1}\left(q^{n-1}+q^{n/2}\right)+q^{n/2} \\
&\geq& 1-2q^{n/2}\left(1+q^{\frac{n-2}{2}}\right)^{n-1}+q^{n/2}
\end{eqnarray*}

Bounding the middle term using Lemma~\ref{lemma:middle} (found in the appendix, Section~\ref{sec:appendixA}), we then find that for all $n\geq 4$,
\begin{eqnarray*}
P\left({\cal G}_{n,p} \mbox{ is connected}\right)&\geq& 1-20q^{n/2}\left(1+q^{n/2}\right)^{n-1}+q^{n/2}\\
&\geq& 1-q^{n/2}\left(20\left(1+q^{n/2}\right)^{n-1}-1\right)\\
\end{eqnarray*}
Considering the entire set of $N$ items, there will be at most $\frac{N}{n}$ leaf clusters to resolve (where each cluster has $\geq n$ items).  Therefore, we bound the probability that ${\cal G}_{n,p}$ is disconnected as,
\begin{eqnarray*}
q^{n/2}\left(20\left(1+q^{n/2}\right)^{n-1}-1\right) \leq \frac{\alpha n}{2N}
\end{eqnarray*}
Where, using the union bound, we state that all $\leq \frac{N}{n}$ clusters containing $\geq n$ items will be connected with probability $\geq 1-\frac{\alpha}{2}$.

Bounding $q^{n/2} \leq C$, with $C>0$ being a dummy variable, then
\begin{eqnarray*}
q^{n/2}\left(20\left(1+q^{n/2}\right)^{n-1}-1\right) \leq C\left(20\left(1+q^{n/2}\right)^{n-1}-1\right) \leq \frac{\alpha n}{2N}
\end{eqnarray*}

Therefore, we can solve with respect to the probability of observation, $p$, and dummy variable $C$ that,
\begin{eqnarray}
\label{eqn:boundA}
q^{n/2} &\leq& C \notag \\
p &\geq& 1-C^{2/n}
\end{eqnarray}
And,
\begin{eqnarray*}
C\left(20\left(1+q^{n/2}\right)^{n-1}-1\right) &\leq& \frac{\alpha n}{2N} \\
\end{eqnarray*}

Then rearranging these terms with respect to $q$,
\begin{eqnarray*}
q &\leq& \left(\left(\frac{\alpha n}{40NC} + \frac{1}{20}\right)^{\frac{1}{n-1}} - 1 \right)^{2/n}\\
\end{eqnarray*}
In terms of the probability of pairwise similarity observation ($p$, where $p = 1-q$),
\begin{eqnarray}
\label{eqn:boundB}
p &\geq& 1-\left(\left(\frac{\alpha n}{40NC} + \frac{1}{20}\right)^{\frac{1}{n-1}} - 1 \right)^{2/n}
\end{eqnarray}

Combining the bounds in Equations~\ref{eqn:boundA} and~\ref{eqn:boundB}, we find that for a given choice of $C\geq0$,
\begin{eqnarray}
\label{eqn:boundC}
p\geq \max\left\{1-C^{2/n}, 1-\left(\left(\frac{\alpha n}{40NC} + \frac{1}{20}\right)^{\frac{1}{{n-1}}} - 1 \right)^{2/n}\right\}
\end{eqnarray}
We note that for any choice of $\alpha,n,N$, the first term is monotonically decreasing in $C$, while second term is monotonically increasing in $C$.  To simplify the analysis, we take,
\begin{eqnarray}
C = \frac{\alpha n}{78 N}
\end{eqnarray}
Plugging this value of C into Equation~\ref{eqn:boundC} gives us the result.

\subsection{Proof of Theorem~\ref{thm:constantThm}}
Using Theorem~\ref{thm:global}, we find the required probability of similarity observation for a given $n,N$.  Given that $n \geq \delta N^\beta \geq 4$, we can state,
\begin{eqnarray}
\label{eqn:appendixP}
p&\geq& \max\left\{1-\left(\frac{\alpha \delta N^\beta}{78N}\right)^{2/\left({\delta N^\beta}\right)}, 1-\left(2^{\frac{1}{\delta N^\beta-1}} - 1 \right)^{2/\left({\delta N^\beta}\right)}, 1-\left(\frac{\alpha}{2\left(N-\delta N^\beta\right)}\right)^{1/\left({\delta N^\beta}\right)}\right\}
\end{eqnarray}

We now show that the specified sampling probability rate, $\frac{2 \kappa}{\delta}N^{-\beta}\log{N}$, is greater than all terms inside this bound.

\subsubsection{First Term Derivation}
The first term of Equation~\ref{eqn:appendixP} is satisfied if,
\begin{eqnarray*}
\frac{2 \kappa}{\delta}N^{-\beta}\log{N} &\geq& 1-\left(\frac{\alpha \delta N^\beta}{78N}\right)^{2/\left({\delta N^\beta}\right)} \\
\frac{{\delta N^\beta}}{2}\log\left(1-\frac{2 \kappa}{\delta}N^{-\beta}\log{N}\right) &\leq& \log\left(\frac{\alpha \delta N^{\beta-1}}{78}\right)
\end{eqnarray*}
Bounding the logarithm function, we can state that $\log\left(1-\frac{2 \kappa}{\delta}N^{-\beta}\log{N}\right) \leq \frac{-2 \kappa}{\delta}N^{-\beta}\log{N}$.  Therefore,
\begin{eqnarray*}
\frac{{\delta N^\beta}}{2}\left(\frac{-2 \kappa}{\delta}N^{-\beta}\log{N}\right) &\leq&  \log\left(\frac{\alpha \delta N^{\beta-1}}{78}\right) \\
\end{eqnarray*}
Then,
\begin{eqnarray*}
\left(1-\beta-\kappa\right)\log{N} &\leq&  \log\left(\frac{\alpha \delta}{78}\right)
\end{eqnarray*}
Therefore, this bound holds if $N \geq \left(\frac{\alpha \delta}{78}\right)^{{1}/{\left(1-\beta-\kappa\right)}}$.

\subsubsection{Second Term Derivation}
The second term,
\begin{eqnarray*}
\frac{2 \kappa}{\delta}N^{-\beta}\log{N} &\geq& 1-\left(2^{\frac{1}{\delta N^\beta-1}} - 1 \right)^{2/\left({\delta N^\beta}\right)} \\
\frac{\delta N^\beta}{2}\log\left(1-\frac{2 \kappa}{\delta}N^{-\beta}\log{N}\right) &\leq& \log\left(2^{\frac{1}{\delta N^\beta-1}} - 1 \right)
\end{eqnarray*}
Again, bounding the log function,
\begin{eqnarray*}
\frac{\delta N^\beta}{2}\left(-\frac{2 \kappa}{\delta}N^{-\beta}\log{N}\right) &\leq& \log\left(2^{\frac{1}{\delta N^\beta-1}} - 1 \right)\\
-\kappa\log{N} &\leq& \log\left(2^{\frac{1}{\delta N^\beta-1}} - 1 \right)\\
\end{eqnarray*}
We can then lower bound the term $\log\left(2^{\frac{1}{\delta N^\beta-1}} - 1 \right)$.  Given that $2^{\frac{1}{\delta N^\beta-1}} - 1 \in \left(\frac{1}{N},1\right)$ for any $\delta N^\beta > 1$,
\begin{eqnarray*}
\log\left(2^{\frac{1}{\delta N^\beta-1}} - 1 \right) &\geq& \frac{\log{1} - \log{\frac{1}{N}}}{1 - \frac{1}{N}}\left(2^{\frac{1}{\delta N^\beta-1}}-2\right) \\
&=& \frac{N\log{\left(N\right)}}{N-1}\left(2^{\frac{1}{\delta N^\beta-1}}-2\right)
\end{eqnarray*}

Then,
\begin{eqnarray*}
-\kappa\log{N} &\leq& \frac{N\log{\left(N\right)}}{N-1}\left(2^{\frac{1}{\delta N^\beta-1}}-2\right) \\
\end{eqnarray*}
And rearranging this term,
\begin{eqnarray*}
-\kappa + \frac{2N}{N-1} &\leq& \frac{N}{N-1}\left(2^{\frac{1}{\delta N^\beta-1}}\right) \\
\end{eqnarray*}
Given that the right hand side is always greater than zero, we can then find that the second term holds if,
\begin{eqnarray*}
\kappa &\geq& \frac{2N}{N-1}
\end{eqnarray*}
Therefore, the second term holds if $\kappa\geq 3$ and $N\geq 4$.

\subsubsection{Third Term Derivation}
And finally, the third term can be bounded as,
\begin{eqnarray*}
\frac{2 \kappa}{\delta}N^{-\beta}\log{N} &\geq& 1-\left(\frac{\alpha}{2\left(N-\delta N^\beta\right)}\right)^{1/\left({\delta N^\beta}\right)} \\
\left({\delta N^\beta}\right)\log\left({1-\frac{2 \kappa}{\delta}N^{-\beta}\log{N}}\right) &\leq& \log\left(\frac{\alpha}{2\left(N-\delta N^\beta\right)}\right)
\end{eqnarray*}
Bounding the log function,
\begin{eqnarray*}
\left(-{\delta N^\beta}\right)\frac{2 \kappa}{\delta}N^{-\beta}\log{N} = -2\kappa\log{N} &\leq& \log\left(\frac{\alpha}{2\left(N-\delta N^\beta\right)}\right)
\end{eqnarray*}
Then,
\begin{eqnarray*}
N^{-2\kappa} &\leq& \frac{\alpha}{2\left(N-\delta N^\beta\right)} \\
N^{-2\kappa}\left(N-\delta N^\beta\right) = N^{1-2\kappa}\left(1-\delta N^{\beta-1}\right) \leq N^{1-2\kappa} &\leq& \frac{\alpha}{2}
\end{eqnarray*}
Then we find that this term is satisfied if $N\geq \left(\frac{\alpha}{2}\right)^{1/\left(1-2\kappa\right)}$

By combining all three bounds, we find that sampling with probability $\frac{2 \kappa}{\delta}N^{-\beta}\log{N}$ will resolve a pruning of the hierarchical clustering down to clusters of size $\geq \delta N^{\beta} \geq 4$, if $\kappa\geq 3$ and $N\geq\max\{4,\left(\frac{\alpha}{2}\right)^{1/\left(1-2\kappa\right)},\left(\frac{\alpha \delta}{78}\right)^{{1}/{\left(1-\beta-\kappa\right)}}\}$.

\bibliographystyle{IEEEtran}
\bibliography{nipsCiteNew}

\end{document}